\documentclass[11pt,a4paper]{article}

\usepackage{times}

\usepackage{empheq,amsthm,amssymb}

\usepackage{mathrsfs}

\usepackage{xcolor}
\usepackage{graphicx,caption,subcaption}
\graphicspath{{FIGS/}}

\usepackage{hyperref}
\hypersetup{
	colorlinks   = true, 
	urlcolor     = blue, 
	linkcolor    = teal, 
	citecolor   = olive
}

\usepackage[numbers, compress]{natbib}

\usepackage[margin=2.55cm]{geometry}

\usepackage{enumitem}
\usepackage[numbered,framed]{matlab-prettifier}

\numberwithin{equation}{section}

\newtheorem{theorem}{Theorem}
\newtheorem{lemma}{Lemma}
\newtheorem{corollary}{Corollary}

\newcommand{\vertiii}[1]{{\left\vert\kern-0.25ex\left\vert\kern-0.25ex\left\vert\, #1 
    \,\,\right\vert\kern-0.25ex\right\vert\kern-0.25ex\right\vert}}

\begin{document}

\title{Universal Time-Uniform Trajectory Approximation for Random Dynamical Systems with Recurrent Neural Networks}
\author{Adrian N. Bishop}
\date{}

\maketitle

\begin{abstract}
The capability of recurrent neural networks to approximate trajectories of a random dynamical system, with random inputs, on non-compact domains, and over an indefinite or infinite time horizon is considered. The main result states that certain random trajectories over an infinite time horizon may be approximated to any desired accuracy, uniformly in time, by a certain class of deep recurrent neural networks, with simple feedback structures. The formulation here contrasts with related literature on this topic, much of which is restricted to compact state spaces and finite time intervals. The model conditions required here are natural, mild, and easy to test, and the proof is very simple.
\end{abstract}


\section{Introduction}

Universal approximation results for feedforward neural networks \cite{Haykin2009a} state that certain function classes can be approximated to any accuracy by large networks of simple compositions of simple activation functions. The survey \cite{Pinkus1999a} reviews seminal work on the approximation capabilities of feedforward neural networks with a single hidden layer of arbitrary width and arbitrary activations. Universal approximation with networks of fixed width but arbitrary depth is studied in \cite{HaninSellke2017a,LuPuWangEtAl2017a,KidgerLyons2020a,ParkYunLeeEtAl2021a}. Efficiency/optimality of approximation is studied in \cite{Barron1994a,Telgarsky2016a,Yarotsky2018a,ElbrachterPerekrestenkoGrohsEtAl2021a,WangQu2022a}. Differing function classes, domains, error rates, network structures, activation functions, etc, are considered. 

 A recurrent neural network (RNN) incorporates feedback from the output of some activation functions to the input of others with, say, a unit time delay \cite[e.g. Figures 1 and 2]{Elman1990a}. Thus, RNN's can model dynamical systems with feedback and external inputs; i.e. iterations of a map of an input \textit{and} the state of the system at a prior time, to the output state at the current time \cite{Elman1990a,Haykin2009a}. This article concerns only discrete-time RNNs, e.g. as in \cite{PascanuGulcehreChoEtAl2014a}. 

We pause to distinguish two high-level concepts of approximation accuracy in the context of RNNs. The first concerns \textit{map approximation}: in this case, the RNN is frozen in time and the accuracy of the map from the prior time state of the network and the input space, to the output space or current state of the network is appraised (relative to some ground truth). Generally, map approximation is just a reformulation of the approximation theory of feedforward neural networks \cite{Doya1993a}. The second concerns \textit{trajectory approximation}: in this case, the accuracy of the RNN's output trajectory, on some interval, relative to some ground truth trajectory is appraised. Accuracy in this case is a product of both the map approximation error and the accumulation of that error over time due to iterated compositions. 

The capability of RNNs to approximate discrete-time state-space models and trajectories on finite time intervals is studied in \cite{Li1992a,Sontag1992a,Doya1993a,JinNikiforukGupta1995a,SchaeferZimmermann2007a,Patan2008a,KimPatronBraatz2011a}. Map approximation on its own is studied in \cite{Doya1993a}; but such results follow directly from the approximation capabilities of feedforward networks. The article \cite{JinNikiforukGupta1995a} studies the approximation of certain trajectories of state-space models with bounded inputs. The idea in \cite{JinNikiforukGupta1995a} relies strictly on the assumption of a finite time interval of interest. In that case, the trajectory is constrained to a compact set and a continuous map may be assumed to be Lipschitz. Any error due to map approximation then grows at most in time (i.e. through repeated map iterations) by a factor determined by the maximum approximation error in a single-step (taken over a compact set) and the iterated Lipschitz constant of the map. By picking the single-step map approximation error small enough, the accumulated error can be bounded below a desired threshold when the horizon is given and finite. The results in \cite{SchaeferZimmermann2007a,Patan2008a,KimPatronBraatz2011a} are mostly the same. 

 In this article we study \textbf{universal \textit{time-uniform} trajectory approximation}: i.e. we consider the properties of very general (random) dynamical systems and the RNN characteristics that enable trajectory approximation to any desired accuracy, uniformly in time, over an indefinite or infinite time interval. The trick in this case is to ensure the map approximation errors are uniformly controlled in space and the error incurred at each time step is forgotten fast enough so that the total error does not accumulate unbounded. It is really the imposition of conditions on the dynamical system to be approximated that facilitate the result in this article. 

The work in \cite{Matthews1993a,MatthewsMoschytz1994a} considered the problem of time-uniform trajectory approximation with RNNs for deterministic discrete-time dynamical systems on compact domains with bounded inputs. A stability condition is imposed on the system approximated such that the approximation error incurred at any single time is forgotten fast enough so that the sum of errors does not grow unbounded. The condition imposed on the system in \cite{Matthews1993a} is a type of so-called fading memory condition, see \cite{BoydChua1985a}; and it is essentially a type of contractive stability condition. See also \cite{Sandberg1991a} where stronger conditions are imposed. Fading memory-type conditions were later applied to derive time-uniform trajectory approximation results in the setting of so-called echo state networks (i.e. RNNs in which most hidden layer parameters are chosen at random and only the output layer parameters are learnt in training), see \cite{GrigoryevaOrtega2018b,GrigoryevaOrtega2018a,HartHookDawes2020a,GononOrtega2020a,GononOrtega2021a}. 

Approximation of continuous-time trajectories (e.g. ordinary differential equations) with continuous time RNNs is studied in \cite{Sontag1992a,FunahashiNakamura1993a,KosmatopoulosPolycarpouChristodoulouEtAl1995a,ChowLi2000a,LiHoChow2005a}. That work is analogous to \cite{JinNikiforukGupta1995a,SchaeferZimmermann2007a}; assuming a fixed finite time horizon, trajectories restricted to a compact set, and a map approximation error small enough so that the accumulated error (which grows, typically exponentially) remains below a desired threshold over the finite horizon. Time-uniform approximation of stable ordinary differential equations with bounded inputs using continuous-time RNNs is introduced in \cite{HansonRaginsky2020a}; with trajectories restricted to a compact set. A special case of continuous-time approximation of periodic orbits is considered in \cite{NakamuraNakagawa2009a}. 

With regards to time-uniform approximation, we comment further on \cite{Matthews1993a}, the work in echo state networks \cite{GrigoryevaOrtega2018a,GononOrtega2020a}, and the continuous-time setting in \cite{HansonRaginsky2020a}, in a later section; after which the notation, certain technical concepts, and the main result and its proof can be easily referenced. 

There are considerable differences in the approach taken in this article, when compared with existing literature. We remark on some characteristics of the formulation, methods, and results in this article worth noting here. We consider discrete-time dynamical systems and trajectories on non-compact spatial domains and over indefinite/infinite time intervals. We consider random dynamical systems with (random) unbounded inputs. We consider approximation in the natural sense of $p$-moments. This theory does not preclude the special case of deterministic state-space models, or trajectories evolving in some compact set. However, it is more natural in some common applications of RNN-based trajectory approximation to consider the formulation herein. The conditions imposed are minimal, testable, and the proof is simple and transparent in its application of the conditions so as to yield some pedagogical value. The formulation and presentation is framed in the language of (random) dynamical systems and state-space models. We provide examples, both empirical and derived, that illustrate the natural satisfaction of the assumptions and motivate the result. As a side project, we also give a different network construction for universal approximation of finite length trajectories that complements \cite{JinNikiforukGupta1995a,SchaeferZimmermann2007a}.

We postpone further discussion until later; when more detail can be referenced. Later we discuss the practical relevance (in design/training) and applicability of this work; we comment on the main result and its relationship to the perturbation theory of dynamical systems, and to related work \cite{Matthews1993a,GrigoryevaOrtega2018a,GononOrtega2020a,HansonRaginsky2020a}; we touch on the approximation of Lipschitz maps; and we contrast this work with trajectory approximation on finite time intervals and the relevance of this distinction; and we discuss extensions.

\section{Dynamical Systems}

Let $\mathbb{X}\subseteq\mathbb{R}^{d_x}$ and let $\|\cdot\|:\mathbb{X}\rightarrow[0,\infty)$ be any norm on $\mathbb{X}$. Let $(\mathbb{X}, \mathcal{B}(\mathbb{X}))$ be a standard Borel space. Let $t\in\mathbb{N}:=\{0,1,2,\ldots\}$. Consider the (Borel) measurable endomorphism $f:\mathbb{X}\rightarrow\mathbb{X}$. A dynamical system \cite{CornfeldFominSinai1982a} is just the compositional iteration of $f$, acting on some initial point $x_0\in\mathbb{X}$, that is,
\begin{equation}
	x_t \,=\,  f(x_{t-1}) \,=\,  f^t(x_0)  \label{recursiveSysDef}
\end{equation}
with $f^0=\mathrm{Id}$; and composition often written as multiplication. The solution $\{f^t(x_0) \,|\,t\in\mathbb{N}\}$ traces a sequence of points in $\mathbb{X}$ which we call the \textit{trajectory} of the dynamical system from $x_0\in\mathbb{X}$. 

The Polish space of all probability measures on $\mathbb{X}$ is denoted by $\mathcal{P}(\mathbb{X})$. On $(\mathcal{P}(\mathbb{X}), \mathcal{B}(\mathcal{P}(\mathbb{X}))$ we have the push-forward operation $f_\#\mu(A) := \mu(f^{-1}(A))$, $\forall A\in\mathcal{B}(\mathbb{X})$ with $f^0_\#:=\mathrm{Id}$. Treating measures $\mu_t\in\mathcal{P}(\mathbb{X})$ as points in a Polish space, we write, as in (\ref{recursiveSysDef}), the shorthand for the dynamical system,
\begin{equation}
	\mu_t \,=\, f_\#\mu_{t-1} \label{recursiveSysMeas}
\end{equation}
and consider the trajectory $\{f^t_\#\mu_{0}  \,|\,t\in\mathbb{N}\}$ of points in $\mathcal{P}(\mathbb{X})$ for a sequence $t\in\mathbb{N}$.

An \textit{invariant measure} for $f$ is any $\mu\in\mathcal{P}(\mathbb{X})$ with $\mu(f^{-1}(A)) = \mu(A)$ for all $A\in\mathcal{B}(\mathbb{X})$. An invariant measure is a fixed-point trajectory $\mu = f^t_\#\mu$ for all $t\in\mathbb{N}$ in $\mathcal{P}(\mathbb{X})$. Fixing the space $(\mathbb{X},\mathcal{B},\mu)$ first, a dynamical system is called \textit{measure-preserving} if $\mu = f^t_\#\mu$. See \cite{CornfeldFominSinai1982a} for general background.

\subsection{Random Dynamical Systems}

Fix throughout the standard probability space $(\Omega,\mathcal{B}(\Omega),\mathsf{P})$. Consider the $\mathfrak{L}{\mathrm{ip}}(\mathbb{X})$-valued\footnote{We denote by $\mathcal{L}\mathrm{ip}(\mathbb{X})$ the space of Lipschitz endomorphisms on $\mathbb{X}$ with a finite Lipschitz constant.} random process defined by the indexed family $\{F_t \,|\, t\in\mathbb{N}\}$ of measurable maps $F_t : \Omega\rightarrow\mathfrak{L}{\mathrm{ip}}(\mathbb{X})$. We write $F_t(\omega)(\cdot):\mathbb{X}\rightarrow\mathbb{X}$ for a realised $\omega\in\Omega$ sample function. This process is well defined, see e.g. \cite{DiaconisFreedman1999a}.

A random dynamical system on $\mathbb{X}$ is defined by the compositional iteration of random (Lipschitz) endomorphisms $F_t$ acting on some initial $x_0\in\mathbb{X}$,
\begin{equation}
	X^{x_0}_t \,=\,  F_{t} \cdots  F_2 \cdot F_1(x_0) \,=\,  F_{t}(X^{x_0}_{t-1}) \label{recursiveRandSysDef}
\end{equation}
The indexed family $\{X_t^{x_0} \,|\,t\in\mathbb{N},x_0\in\mathbb{X} \}$ of random variables $X^{x_0}_t:\Omega\rightarrow\mathbb{X}$ is an induced random process\footnote{We write $X^{\mu_0}_t$ if $X_0:\Omega\rightarrow\mathbb{X}$ is random with distribution $\mu_0\in\mathcal{P}(\mathbb{X})$. In general we write $X'\sim\mu$ if $X'$ is random with law $\mu\in\mathcal{P}(\mathbb{X})$. We write $X^{X'}_t$ if $X_0=X'$ almost surely. We often drop reference to the initial condition.} on $\mathbb{X}$.  The realised solution $\{X_t^{x_0}(\omega) \,|\,t\in\mathbb{N}\}$ is the \textit{trajectory} of the random dynamical system from the state $x_0\in\mathbb{X}$ with $\omega\in\Omega$.

Let $\mathbb{U}\subseteq\mathbb{R}^{d_u}$. Let $(\mathbb{U}, \mathcal{B}(\mathbb{U}))$ be a standard Borel space. The indexed family $\{U_t \,|\, t\in\mathbb{N}\}$ of random variables $U_t:\Omega\rightarrow\mathbb{U}$ is a random process. We consider the maps in (\ref{recursiveRandSysDef}) of the form,
\begin{equation}
	F_{t}(\omega)(\cdot) \,:=\, f(\cdot,U_{t}(\omega))  \label{recursiveRandSysInputDef}
\end{equation}
where $f:\mathbb{X}\times\mathbb{U}\rightarrow\mathbb{X}$ is a deterministic, time-invariant, measurable map, Lipschitz in $\mathbb{X}$.

See \cite{Ohno1983a,ArnoldCrauel1992a,DiaconisFreedman1999a,Arnold2003a} for further background on these systems. Note that an independent and identically distributed sequence of maps does not induce an independent process in $X_t$, but it is Markov \cite{Kifer1986a}; and, a stationary Markov sequence of maps does not induce a Markov process in $X_t$. A stationary process of certain contractive maps is, with the right initialisation, a stationary process in $X_t$ \cite{Elton1990a}.

\begin{lemma}\cite{Kifer1986a}\label{markoviidlemma}
	The process $\{X_t^{x_0}\,|\,t\in\mathbb{N}\}$ in (\ref{recursiveRandSysDef}) is Markov for all $x_0\in\mathbb{X}$ if and only if the sequence of random maps $\{F_t \,|\, t\in\mathbb{N}\}$ is independently distributed. A Markov process $\{X_t^{x_0} \,|\,t\in\mathbb{N}\}$ in (\ref{recursiveRandSysDef}) is homogenous for all $x_0\in\mathbb{X}$ if and only if the sequence $\{F_t \,|\, t\in\mathbb{N}\}$ is identically distributed. 
\end{lemma}

Consider an invertible measure-preserving endomorphism $\theta:\Omega\rightarrow\Omega$ on $(\Omega,\mathcal{B}(\Omega),\mathsf{P})$. If $F_t(\omega) = F(\theta^{t}(\omega))$ for a random $F : \Omega\rightarrow\mathfrak{L}{\mathrm{ip}}(\mathbb{X})$, then $\{F_t \,|\, t\in\mathbb{N}\}$ is stationary \cite[Appendix A]{Arnold2003a}. An extended\footnote{It is possible to extend a sequence of stationary random maps $\{F_t \,|\, t\in\mathbb{N}\}$ over all the integers $\mathbb{Z}$ such that the doubly infinite process $\{F_t \,|\, t\in\mathbb{Z}\}$ is stationary and the law of any combination of maps with non-negative time indices in the extended sequence is the same as in the original sequence; e.g. see \cite[Lemma 1]{Elton1990a}, Doob \cite[page 456]{Doob1953a}, or \cite[Appendix A]{Arnold2003a}.} stationary sequence $\{F_t \,|\, t\in\mathbb{Z}\}$ is identically distributed (but not vice versa).

The Lipschitz constant of $F\in\mathfrak{L}{\mathrm{ip}}(\mathbb{X})$ is the smallest $\mathfrak{L}{(F)}:\Omega\rightarrow[0,\infty)$ such that,
\begin{equation}
	\|F(\omega)(x) - F(\omega)(x_0 )\| \,\leq\, \mathfrak{L}{(F)}(\omega)\,\|x - x_0\|
\end{equation}
A particular Lipschitz map is called \textit{contractive} if its Lipschitz constant is $<1$. Loosely speaking, if the sequence $\{F_t \,|\, t\in\mathbb{Z}\}$ is stationary and over time the maps are contracting in some average sense, then the induced process $\{X_t\,|\,t\in\mathbb{N}\}$ converges in distribution to a stationary process. The induced process also exhibits some contractive stability and initial states are forgotten. For example, if $\{F_t \,|\, t\in\mathbb{Z}\}$ is independent and identically distributed, then the condition $\mathsf{E}[\log\mathfrak{L}{(F)}]<0$ formalises a notion of \textit{contracting in the average}, and leads to the stated conclusions; see \cite{DiaconisFreedman1999a}. We will not generally impose an independence assumption on the maps, i.e. we do not assume $\{X_t\,|\,t\in\mathbb{N}\}$ is Markov. We also assume a weaker notion, in some directions at least, of contracting in the average. The following result is an extension of a classical result of Elton \cite{Elton1990a}, see also \cite{DebalyTruquet2021a}.

\begin{lemma}\label{eltonlemma}
Assume $\{F_t(\omega) = F(\theta^{t}(\omega)) \,|\, t\in\mathbb{Z}\}$ is stationary and $\mathsf{E}[\|F(x)-x\|^p]<\infty$ for one $x\in\mathbb{X}$ and some $p\geq1$. Assume
there is a finite $C>0$, and $\lambda\in(0,1)$, all independent of $s,t\in\mathbb{N}$, such that,
\begin{align}
	\mathsf{E} \big[ \,\| F_{t}  \cdots  F_{s+1}(X^{x}_s) - F_{t}  \cdots  F_{s+1}({X}^{x_0}_s) \|^p \,\big] 
	  ~&\leq~ C\,\lambda^{(t-s)}\, \mathsf{E} \big[\, \|{X}^{x}_{s} - {X}^{x_0}_{s} \|^p \,\big] \label{stabilitycond}
\end{align}
exists for all $x,x_0\in\mathbb{X}$ and all $t> s$. Then there is a random variable, $X_\infty:\Omega\rightarrow\mathbb{X}$, such that,
 \begin{equation}
		\lefteqn{\overbrace{\phantom{\lim_{t\rightarrow\infty}~X_t^{x_0} \,=\, \lim_{t\rightarrow\infty} ~F_{t} \cdots  F_2 \cdot F_1(x_0) ~\stackrel{\text{law}}{=}~ X_\infty}}^{forward~convergence~in~distribution}}\lim_{t\rightarrow\infty}~X_t^{x_0} \,=\, \lim_{t\rightarrow\infty} ~F_{t} \cdots  F_2 \cdot F_1(x_0) ~\stackrel{\text{law}}{=}~ \underbrace{X_\infty~=~ \lim_{t\rightarrow\infty} ~F_{0}\cdot F_{-1}  \cdots  F_{-t}(x_0)}_{backward~almost~sure~convergence} \label{convergencelemmaeqn}
\end{equation} 
and $X_\infty$ is independent of $x_0$, i.e, $\mathsf{P}(X_\infty(x_0)\neq X_\infty(x))=0$. The process $\{X_t^{\mu_0}\,|\,t\in\mathbb{N}\}$ with $\mu_0:=\mathrm{Law}(X_{\infty})$ is stationary in $\mathbb{X}$. If $\{F_t \,|\, t\in\mathbb{N}\}$ is also ergodic, then $\{X_t \,|\, t\in\mathbb{N}\}$ is ergodic. 
\end{lemma}

The condition of interest is (\ref{stabilitycond}) which asks for the composition of maps to \textit{eventually be contracting on average}; e.g. see the terminology in \cite{Steinsaltz1999a,WuShao2004a}. If $\{F_t \,|\, t\in\mathbb{N}\}$ is also independent, then $\{X_t^{x_0}\,|\,t\in\mathbb{N}\}$ is Markov and the backward limit is given by $X_\infty = \lim_{t\rightarrow\infty} ~F_{1}\cdot F_{2}  \cdots  F_{t}(x_0)$, see \cite{DiaconisFreedman1999a}. Variations on the type of average contraction, e.g. on \eqref{stabilitycond}, that lead (mostly) to the same conclusions are studied in the literature, particularly when $\{X_t\,|\,t\in\mathbb{N}\}$ is Markov, or when $F_t$ is affine or takes values in a finite subset of $\mathfrak{L}{\mathrm{ip}}(\mathbb{X})$. See \cite{BarnsleyElton1988a,Elton1990a, ArnoldCrauel1992a,BurtonRoesler1995a,Steinsaltz1999a,WuShao2004a,Stenflo2012a}. 

A sequence of random maps $\{F_t \,|\, t\in\mathbb{N}\}$, not necessarily stationary, that satisfies \eqref{stabilitycond} is called \textit{exponentially $p$-contractive}.

\subsection{Example: Complex Trajectories from Simple Dynamics} \label{barnsleysfernexample1}

Let $\Omega=\{1,\ldots,K\}^\mathbb{N}$ with finite $K\in\mathbb{N}$, let $\mathsf{p}=(p_1,\ldots,p_K)$, $\sum_ip_i=1$, and let $\mathsf{P}=\mathsf{p}^\mathbb{N}$. Each $\omega\in\Omega$ has representation $\omega=(\omega_0,\omega_1,\ldots)$. Let $\theta:\Omega\rightarrow\Omega$ be the left shift $\theta^t(\omega)=(\omega_t,\omega_{t+1}\ldots)$, defining a measure preserving $\theta^t_\#\mathsf{P}=\mathsf{P}$ dynamical system on $(\Omega,\mathcal{B}(\Omega),\mathsf{P})$. 
Let $\{G_i:\mathbb{X}\rightarrow\mathbb{X}\,|\,i\in\{1,\ldots,K\}\}$ be a collection of deterministic Lipschitz maps. Define a random map by,
\begin{align}
	F: \Omega\rightarrow \{G_1,\ldots,G_K\} \subset \mathfrak{L}{\mathrm{ip}}(\mathbb{X}) \qquad\mathrm{with}\qquad
	\theta^t(\omega)\mapsto F(\theta^t(\omega)) = G_{\omega_t}
\end{align}
The sequence $\{F_t(\omega) = F(\theta^{t}(\omega)) \,|\, t\in\mathbb{N}\}$ is an independent and identically distributed sequence. For any $x\in\mathbb{X}$, an induced process is generated by independently picking $G_i$ with probability $p_i$ and moving to $G_i(x)$. The induced process $X^{x_0}_t = F_{t} \cdots  F_2 \cdot F_1(x_0)$, see (\ref{recursiveRandSysDef}), is a Markov chain. 

Remarkably complicated and interesting induced trajectory behaviour may arise from this seemingly simple system, even in low dimension with $K$ small. For example, let $\mathbb{X}=\mathbb{R}^2$, $K=4$ and,
\begin{align}
	G_1(x) \,=\, \begin{bmatrix}\ 0.00&\ 0.00\ \\0.00&\ 0.16\end{bmatrix}x+{\begin{bmatrix}\ 0.00\\0.00\end{bmatrix}},\quad G_2(x) \,=\, {\begin{bmatrix}\ 0.85&\ 0.04\ \\-0.04&\ 0.85\end{bmatrix}}x+{\begin{bmatrix}\ 0.00\\1.60\end{bmatrix}}, \\
	G_3(x) \,=\,  {\begin{bmatrix}\ 0.20&\ -0.26\ \\0.23&\ 0.22\end{bmatrix}}x+{\begin{bmatrix}\ 0.00\\1.60\end{bmatrix}}, \quad G_4(x) \,=\, {\begin{bmatrix}\ -0.15&\ 0.28\ \\0.26&\ 0.24\end{bmatrix}}x+{\begin{bmatrix}\ 0.00\\0.44\end{bmatrix}}
\end{align}
with $\mathsf{p}=(0.01, 0.85, 0.07, 0.07)$. Lemma \ref{eltonlemma} trivially applies here. Set $x_0=0$ and iterate the induced random process over, say, $T=1000$, $T=10000$ and $T=100000$ time-steps. The resulting empirical trajectory in phase space is plotted in Figure \ref{fig:barnsleysfern}, and is known as Barnsley's fern \cite{BarnsleyElton1988a}.

\begin{figure}[!ht]
    \centering
    \includegraphics[width=0.75\textwidth]{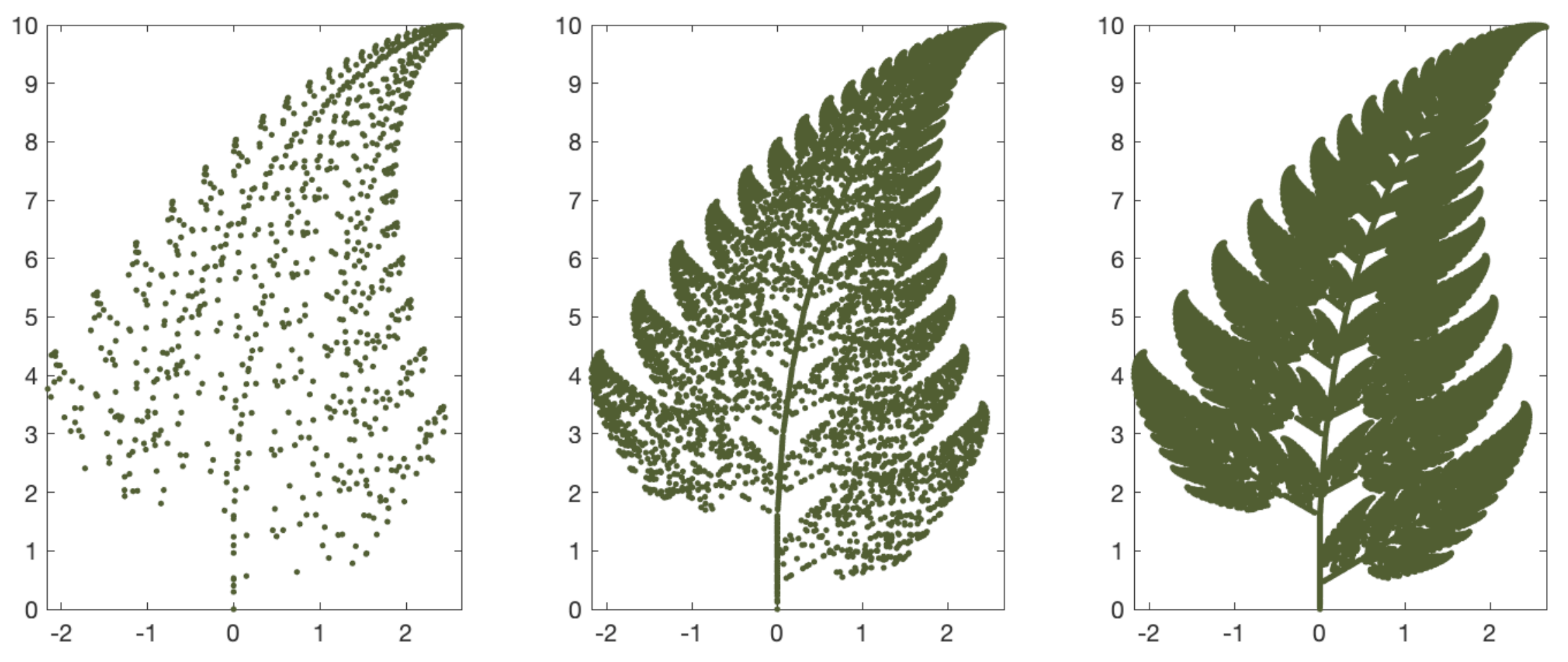}
    \caption{The trajectory of Barnsley's fern process with $T=1000$, $T=10000$ and $T=100000$.}
    \label{fig:barnsleysfern}
\end{figure}

We do not attempt to survey the rich and interesting nature of this type of random dynamical system or its ergodic and invariant measure properties; see e.g. \cite{Hutchinson1981a,BarnsleyElton1988a}. We simply note that the forward induced process generating Barnsley's fern does not converge to any point (obviously) but instead wanders ergodically through the state space in accordance with its unique invariant measure (which happens to place most mass in the shape of the depicted fern, e.g. as in the right-most image in Figure \ref{fig:barnsleysfern}).

\subsection{Example: Contraction on Average with Non-Contractive Maps}

Assume the independent setup in Section \ref{barnsleysfernexample1} and consider two non-contractive maps,
\begin{align}
	G_1(x) \,=\, \begin{bmatrix}\ 10/9 &\ 0\\ \ 0&\ 1/2\end{bmatrix}x +\mathbf{b}_1 \qquad G_2(x) \,=\, {\begin{bmatrix}\ 1/2&\ 0\ \\0&\ 10/9\end{bmatrix}}x +\mathbf{b}_2
\end{align}
with $\mathsf{p}=(0.5, 0.5)$ and any $\mathbf{b}_1$, $\mathbf{b}_2\in\mathbb{R}^2$. Then for any two initial points $x,x_0\in\mathbb{X}$ consider,
\begin{align}
	\mathsf{E} \big[ \, \|X_1^{x} - X_1^{x_0}\|  \,\big] \,\leq\,   \left(\frac{1}{2}\|\mathbf{A}_1\|\ + \frac{1}{2}\|\mathbf{A}_2\|\right)\|x-x_0\| \,=\, \frac{10}{9}\|x-x_0\|
\end{align}
with $G_i(x)=\mathbf{A}_ix+\mathbf{b}_i$. The map is non-contracting on average in a single iteration. However,
\begin{align}
	\mathsf{E} \big[ \, \|X_2^{x} - X_2^{x_0}\|  \,\big] \,\leq\,   \left(\frac{1}{4}\|\mathbf{A}_1\mathbf{A}_1\|\ +\frac{1}{2}\|\mathbf{A}_1\mathbf{A}_2\|+ \frac{1}{4}\|\mathbf{A}_2\mathbf{A}_2\|\right)\|x-x_0\| \,<\, \frac{9}{10}\|x-x_0\|
\end{align}
and the maps are contracting, on average, over two iterations. Necessary and sufficient conditions, replacing (\ref{stabilitycond}), for the conclusions of Lemma \ref{eltonlemma} to hold in the setting of random affine maps are given in \cite{BougerolPicard1992a}; and take advantage of a logarithm inside the expectation. Here we just provide a basic example for (\ref{stabilitycond}) to hold with each map non-contracting; much more complicated examples are possible. 

This example allows one to contrast the property of exponential $p$-contractivity, an average-type contraction property, with the deterministic fading memory-type property considered in \cite{BoydChua1985a,Matthews1993a,GrigoryevaOrtega2018a}.

\section{Recurrent Neural Network Architectures}

Let $\sigma(\cdot):=\max\{\cdot,0\}:\mathbb{R}\rightarrow\mathbb{H}\subseteq\mathbb{R}$; which acts component wise on a vector $v\in\mathbb{R}^d$ so $\sigma(v)\in\mathbb{H}^d$.

A general deep recurrent neural network (RNN), with $L\in\mathbb{N}$ layers, is given in the form,
\begin{align}
	\widehat{X}_t \,=\, {h}_{L,t} &\,=\, \tau_L({h}_{L-1,t}) \label{RNNform1} \\
	{h}_{l,t} &\,=\, \sigma\left(\tau_l({h}_{l-1,t}) + \varphi_l({h}_{t-1}) \right),\qquad l\in\{1,\ldots,L-1\} \label{RNNform2} \\
	{h}_{0,t} &\,:=\, U_{t}  \label{RNNform3}
\end{align}
where ${h}_{l,t}\in\mathbb{H}^{d_{h_l}}$ with $\mathbb{H}^{d_{h_0}}:=\mathbb{U}$ and $\mathbb{H}^{d_{h_L}}:=\mathbb{X}$. The feedback vector is ${h}_{t}=({h}_{1,t},\ldots,{h}_{L,t})\in\mathbb{H}^{d_h}$. The maps
\begin{equation} 
	\tau_l:\mathbb{H}^{d_{h_{l-1}}}\rightarrow\mathbb{R}^{d_{h_{l}}}  \qquad\mathrm{and}\qquad \varphi_l:\mathbb{H}^{d_h}\rightarrow\mathbb{R}^{d_{h_{l}}}
\end{equation}
are finite affine maps.

We write a generic deep RNN of the form (\ref{RNNform1}), (\ref{RNNform2}), (\ref{RNNform3}) as,
\begin{equation}
	\widehat{X}^{h_0}_t \,=\, \widehat{f}({h}_{t-1}, U_{t}) \,=:\, \widehat{F}_{t}({h}_{t-1}) \label{genericformRNN}
\end{equation} 
Specifying the parameters of the affine maps defines the topology of the network $\widehat{f}$. Specifying also the initial condition $h_0\in\mathbb{H}^{d_h}$ defines the resulting RNN output for a given input sequence. 

The form (\ref{RNNform1}), (\ref{RNNform2}), (\ref{RNNform3}) is very general, covering at least all those models in \cite{PascanuGulcehreChoEtAl2014a}. We will not construct any networks utilising this general form. Common specialisations of this structure (\ref{RNNform1}), (\ref{RNNform2}), (\ref{RNNform3}) involve feedback solely from the last activation layer to the first layer, and feedback from the output of each activation layer to itself, see \cite{PascanuGulcehreChoEtAl2014a}. 

Feedback from the last layer, or equivalently the last hidden layer, in (\ref{RNNform1}), (\ref{RNNform2}), (\ref{RNNform3}) amounts to a recursion on $\widehat{X}_t$. Consider the special form,
\begin{align}
	\widehat{X}_t \,=\, \tau_L\cdot\sigma\cdot\tau_{l-1}\cdots\sigma\cdot\tau_2\cdot\sigma(\tau_1(U_t)  + \varphi(\widehat{X}_{t-1}) ) \label{RNNspecialform}
\end{align}
where $\varphi:\mathbb{X}\rightarrow\mathbb{R}^{d_{h_1}}$ is a finite affine feedback map. This form (\ref{RNNspecialform}) of RNN is written as,
\begin{equation}
	\widehat{X}_t \,=\, \widehat{f}(\widehat{X}_{t-1}, U_{t}) \,=:\, \widehat{F}_{t}(\widehat{X}_{t-1}) \label{genericspecialformRNN}
\end{equation}

\subsection{Map Approximation versus Trajectory Approximation}

Map approximation seeks for any $\epsilon>0$ a network $\widehat{f}({h}, \cdot)$ and a network state ${h}\in\mathbb{H}^{d_h}$ such that,
\begin{equation}
	\mathsf{E}\big[ \,\| \widehat{f}({h},\cdot) -  {f}(x,\cdot) \|^p\, \big]^{1/p} \,\leq\, \epsilon
\end{equation}
for all $x$ and all inputs in $\mathbb{X}\times\mathbb{U}$. The network initialisation is explicitly a part of the approximation. Map approximation theory follows from classical universal approximation results for feedforward networks, e.g. by considering the RNN frozen at a single time. In the proof of our main theorem we call on a map approximation result: we state an approximation theorem for feedforward neural networks in Appendix \ref{UAFFappendix}. In general, repeated composition of an approximate map may lead to an accumulation of error.

The problem of \textit{trajectory approximation} seeks for any $\epsilon>0$ a network $\widehat{f}$ and a network initialisation ${h}_0\in\mathbb{H}^{d_h}$ such that,
\begin{equation}
	\mathsf{E}\big[ \,\big\| \widehat{F}_{t}\cdots\widehat{F}_2\cdot\widehat{F}_1({h}_0) \,-\, {F}_{t}\cdots{F}_2\cdot {F}_1(x_0) \big\|^p\, \big]^{1/p} \,\leq\, \epsilon
\end{equation}
for all times in a given interval $t\in\{0,\ldots,T\}$, $T\in\mathbb{N}$ and for all $x_0$ and all inputs in $\mathbb{X}\times\mathbb{U}^T$. Trajectory approximation over known fixed-length intervals is studied in \cite{JinNikiforukGupta1995a,SchaeferZimmermann2007a}. These results are generally unsuitable for trajectories evolving on non-compact domains and over indefinite/infinite horizons. In Appendix \ref{appendixFiniteTime} we prove another finite-horizon universal trajectory approximation result, based on a significantly different network construction and proof than considered in previous works, e.g. \cite{JinNikiforukGupta1995a,SchaeferZimmermann2007a}. This construction complements prior work and may be of general interest. The result in Appendix \ref{appendixFiniteTime} also accommodates more general dynamical system models than previously considered.

Finally, we call a trajectory approximation error bound that holds $\underline{\underline{\mathrm{for~all}~t\in\mathbb{N}}}$ a \textit{time-uniform} trajectory approximation error bound.

\section{Universal Time-Uniform Trajectory Approximation}

The following is the main result of this article.

\begin{theorem}\label{theoremuniversaltrajapprox}
	Consider the map $f:\mathbb{X}\times\mathbb{U}\rightarrow\mathbb{X}$ of (\ref{recursiveRandSysInputDef}) and a RNN $\widehat{f}$ of the form (\ref{RNNform1}), (\ref{RNNform2}), (\ref{RNNform3}), see also (\ref{genericformRNN}). Let $X_0:\Omega\rightarrow\mathbb{X}$ be random, independent, and have finite $p\geq1$ moments. Assume $\{F_t(\omega) = F(\theta^t(\omega))\,|\, t\in\mathbb{N}\}$ is stationary,  and exponentially $p$-contractive, and $\mathsf{E}\left[\|F(x)-x\|^p\right]<\infty$. Then for any $\epsilon>0$ there is a finite $\widehat{f}$ and a random initial $H_0:\Omega\rightarrow\mathbb{H}^{d_h}$ such that,
\begin{equation}
		\mathsf{E}\big[ \,\big\| \widehat{F}_{t}\cdots\widehat{F}_2\cdot\widehat{F}_1(H_0) \,-\, {F}_{t}\cdots{F}_2\cdot {F}_1(X_0) \big\|^p\, \big]^{{1}/{p}}  ~\leq~ \epsilon
\end{equation}
for all $t\in\mathbb{N}$ and the sequence $\{\smash{\widehat{F}_t}(\cdot):=\smash{\widehat{f}(\cdot,U_t)} \,|\, t\in\mathbb{N}\}$ is exponentially $p$-contractive.
\end{theorem}

The conditions of interest are stationarity and exponential contractivity of $\{F_t(\cdot)=f(\cdot,U_t) \,|\, t\in\mathbb{N}\}$. The induced process $\{X_t\,|\, t\in\mathbb{N}\}$ is not necessarily Markov, but it is asymptotically stationary under the hypotheses of the theorem, via Lemma \ref{eltonlemma}.

\begin{proof}
Consider a RNN of the form (\ref{RNNspecialform}); also written as (\ref{genericspecialformRNN}). This is a (very) special case of the more general (\ref{RNNform1}), (\ref{RNNform2}), (\ref{RNNform3}). Consider the error,
\begin{equation}
		 \widehat{X}_t^{h_{0}} \,-\, X_t^{x_0}  \,=\, \widehat{F}_{t}\cdots\widehat{F}_2\cdot\widehat{F}_1({h_0}) \,-\, {F}_{t}\cdots{F}_2\cdot {F}_1(x_0)  
\end{equation}
for any ${x}_0\in\mathbb{X}$ and with $h_0=h_{L,0}\in\mathbb{X}$ in the special case (\ref{RNNspecialform}). Let $h_{L,0}=x_0$. Then we can expand,
\begin{align}
		 \widehat{X}_t^{{x_0}} \,-\, X_t^{x_0}  \,&=\, \sum_{s=1}^{t} {F}_{t}\cdots{F}_{s+1} \cdot \widehat{F}_{s}\cdot\big(\widehat{F}_{s-1}\cdots\left({x}_0\right)\big) \,-\, {F}_{t}\cdots{F}_{s+1} \cdot {F}_{s}\cdot\big(\widehat{F}_{s-1}\cdots\left({x}_0\right)\big) \nonumber \\
		 &=\, \sum_{s=1}^{t} \widehat{F}_{t}\cdots \widehat{F}_{s+1} \cdot \widehat{F}_{s}\cdot\big({F}_{s-1}\cdots\left({x}_0\right)\big) \,-\, \widehat{F}_{t}\cdots \widehat{F}_{s+1} \cdot {F}_{s}\cdot\big({F}_{s-1}\cdots\left({x}_0\right)\big)
\end{align}
This is a telescopic expansion (easy to verify, e.g. with $t\in\{3,4\}$). Writing $X_{s-1}^{x_0}={F}_{s-1}\cdots({x}_0)$, we can take norms $\|\cdot\|_p:=\mathsf{E}[\|\cdot\|^p]^{1/p}$,
\begin{align}
	\big\| \widehat{X}_t^{{x}_0} - X_t^{x_0} \big\|_p &=\, \big\|  \sum_{s=1}^{t} \widehat{F}_{t}\cdots \widehat{F}_{s+1} \cdot \widehat{F}_{s}\big(X_{s-1}^{x_0}\big) - \widehat{F}_{t}\cdots \widehat{F}_{s+1} \cdot {F}_{s}\big(X_{s-1}^{x_0} \big) \big\|_p \nonumber\\
		&\leq\, \sum_{s=1}^{t}   \,\big\| \widehat{F}_{t}\cdots \widehat{F}_{s+1} \cdot \widehat{F}_{s}\big(X_{s-1}^{x_0}\big) - \widehat{F}_{t}\cdots \widehat{F}_{s+1} \cdot {F}_{s}\big(X_{s-1}^{x_0} \big) \big\|_p \label{normederrordecomposition}
\end{align}
Assume that the sequence of random maps $\{\widehat{F}_t \,|\, t\in\mathbb{N}\}$ is exponentially $p$-contractive. We will verify this assumption later. With this assumption we can write the last line as,
\begin{align}
		\big\| \widehat{X}_t^{{x_0}} - X_t^{x_0} \big\|_p 
		\,&\leq\,  \big\| \widehat{F}_{t}\big(X_{t-1}^{x_0}\big) - {F}_{t}\big(X_{t-1}^{x_0} \big) \big\|_p   \,+\, C\sum_{s=1}^{t-1} \lambda^{(t-s-1)} \big\| \widehat{F}_{s}\big(X_{s-1}^{x_0}\big) - {F}_{s}\big(X_{s-1}^{x_0} \big) \big\|_p \label{sumsinglestep}
\end{align}
This is a sum of single-step approximation errors (arising as a result of the RNN map approximation) weighted by the stability index due to the exponential contractiveness of $\{\smash{\widehat{F}_t} \,|\, t\in\mathbb{N}\}$. 

Let $X'\sim\mu'\in\mathcal{P}(\mathbb{X})$ be some random variable. We will consider $\| \widehat{F}_{s}(X_{s-1}^{x_0}) - {F}_{s}(X_{s-1}^{x_0} ) \|_p$ for any $s\in\{1,\ldots,t\}$. Applying the triangle inequality twice we get,
\begin{align}
		\big\| \widehat{F}_{s}\big(X_{s-1}^{x_0}\big) \,-\, {F}_{s}\big(X_{s-1}^{x_0} \big) \big\|_p \,&\leq\, \big\| \widehat{F}_{s}\big(X_{s-1}^{x_0}\big) \,-\, \widehat{F}_{s}\big(X_{s-1}^{X'} \big) \big\|_p  + \big\| \widehat{F}_{s}\big(X_{s-1}^{X'}\big) \,-\, {F}_{s}\big(X_{s-1}^{x_0} \big) \big\|_p \nonumber \\
		\,&\leq\, \big\| \widehat{F}_{s}\big(X_{s-1}^{x_0}\big) \,-\, \widehat{F}_{s}\big(X_{s-1}^{X'} \big) \big\|_p + \big\| \widehat{F}_{s}\big(X_{s-1}^{X'}\big) \,-\, {F}_{s}\big(X_{s-1}^{X'} \big) \big\|_p \nonumber \\
		& \qquad + \big\| {F}_{s}\big(X_{s-1}^{X'}\big) \,-\, {F}_{s}\big(X_{s-1}^{x_0} \big) \big\|_p   \label{singlestepdecomposition}
\end{align}
Since $\{{F}_t \,|\, t\in\mathbb{N}\}$ and $\{\widehat{F}_t \,|\, t\in\mathbb{N}\}$ are exponentially $p$-contractive (the latter to be verified) we have,
\begin{align}
	\big\| \widehat{F}_{s}\big(X_{s-1}^{x_0}\big) - \widehat{F}_{s}\big(X_{s-1}^{X'} \big) \big\|_p \,+\,  \big\| {F}_{s}\big(X_{s-1}^{X'}\big) -  {F}_{s}\big(X_{s-1}^{x_0} \big) \big\|_p ~ \leq~ C^\dagger\, \lambda^{\dagger^{s}}\,  \big\| x_0 - X' \big\|_p
\end{align}
for all $s\in\{1,\ldots,t\}$ and some finite $C^\dagger>0$, $\lambda^\dagger\in(0,1)$. The term $\| x_0 - X' \|_p$ is finite for all $x_0\in\mathbb{X}$, including those random initialisations with finite $p$-th moments.

Now extend the sequence $\{F_t \,|\, t\in\mathbb{N}\}$ to $\{F_t \,|\, t\in\mathbb{Z}\}$. Let 
\begin{equation}
	X_\infty~=~ \lim_{t\rightarrow\infty} ~F_{0}\cdot F_{-1}  \cdots  F_{-t}(x)
\end{equation}
almost surely from Lemma \ref{eltonlemma} such that $\{X_t^{\mu}\,|\,t\in\mathbb{N}\}$ with $\mu:=\mathrm{Law}(X_{\infty})\in\mathcal{P}(\mathbb{X})$ is stationary in $\mathbb{X}$. 

Setting $\mu'=\mu$, note that $\smash{X_{s}^{X'}}\sim\mu$, $\forall s\in\mathbb{N}$ by stationarity. We then have,
\begin{align}
		\big\| \widehat{X}_t^{{x_0}} - X_t^{x_0} \big\|_p 
		 \,\leq\, \left(1+\frac{C}{1-\lambda}\right) \left( C^\dagger\, \lambda^{\dagger^{t}}\,  \big\| x_0 - X \big\|_p \,+\, \big\| \widehat{f}\big(X,U\big) - {f}\big(X,U \big) \big\|_p \right)
\end{align}
where $X\sim\mu$, and $U:\Omega\rightarrow\mathbb{U}$ is distributed by the invariant law of the stationary $\{U_t(\omega)\,|\, t\in\mathbb{N}\}$.

\textbf{Map Approximation}: If we drop the recursion on $t\in\mathbb{N}$ in (\ref{genericspecialformRNN}), then we may view (\ref{RNNspecialform}) as a feedforward neural network, cf. Appendix \ref{UAFFappendix}, with input $(\widehat{X}_{t-1}, U_{t})\in\mathbb{X}\times\mathbb{U}$. Note the sum of the two affine maps $\tau_1$ and $\varphi$ in (\ref{RNNspecialform}) is just another affine map on $(\widehat{X}_{t-1}, U_{t})\in\mathbb{X}\times\mathbb{U}$. This interpretation is all that is needed for the desired map approximation. In this notation $(\widehat{X}_{t-1}, U_{t}) = (X,U)$. Applying Lemma \ref{lemmaUALp} in Appendix \ref{UAFFappendix} it follows that for any $\overline{\epsilon}>0$ there exists a finite neural network $\widehat{f}$ such that,
\begin{equation}
	 \left\| \,\widehat{f}(X,U) - f(X,U)\,  \right\|_p \,\leq\, \overline{\epsilon} \label{mapapproxtheorem1}
\end{equation}
with the added property that $\mathfrak{L}(\widehat{F})(\omega) \leq \mathfrak{L}({F})(\omega)$, $\forall\omega\in\Omega$, where $\widehat{F}(\cdot):=\widehat{f}(\cdot,U)$, $F(\cdot):=f(\cdot,U)$. This completes a map approximation result. 

Because $\overline{\epsilon}$ can be fixed arbitrary, we may solve, say, $\epsilon \geq (\overline{\epsilon} + C^\dagger\, \| x_0 - X'_0 \|_p)\,(1+\frac{C}{1-\lambda})$ for $\overline{\epsilon}>0$ and we obtain the desired trajectory approximation result for any initialisation $x_0\in\mathbb{X}$, including random initialisations with finite $p$-th moments. We note the accuracy of map approximation needed to achieve the desired trajectory approximation may depend on the initialisation when it is not at stationarity; however, the error from any non-stationary starting point is forgotten exponentially fast $C^\dagger \lambda^{\dagger^{t}} \smash{ \big\| x_0 - X' \big\|_p}$.

It remains to establish the sequence of maps $\{\widehat{F}_t \,|\, t\in\mathbb{N}\}$ is exponentially $p$-contractive. However, this follows directly from the second part of Lemma \ref{lemmaUALp} in Appendix \ref{UAFFappendix} as $\mathfrak{L}(\widehat{F})(\omega) \leq \mathfrak{L}({F})(\omega)$, $\forall\omega\in\Omega$ implies $\{\smash{\widehat{F}_t} \,|\, t\in\mathbb{N}\}$ is exponentially $p$-contractive whenever $\{{F}_t \,|\, t\in\mathbb{N}\}$ is exponentially $p$-contractive.

This completes the proof of Theorem \ref{theoremuniversaltrajapprox}.
\end{proof}

\subsection{Example: Learning to Generate Fractals from Short Samples} \label{sec:fernPractical}

We return to the example and setup in Section \ref{barnsleysfernexample1}. Let $\mathbb{U}=\{1,\ldots,K\}$. A random variable $U:\Omega\rightarrow\mathbb{U}$ is defined by $U:=\omega_0$ where $\omega_0$ is the first coordinate of $\omega\in\Omega$. The probability is $U_{\#}\mathsf{P}=\mathsf{p}\in\mathcal{P}(\mathbb{U})$. Thus, $\mathsf{P}(\{\omega:U(\omega)=i\}=\mathsf{p}(i)=p_i$. The process $\{U_t(\omega) = U(\theta^t(\omega))\,|\, t\in\mathbb{N}\}$ is independent and identically distributed and acts as a switching input  signal for the switched dynamical system. The induced process from the setup in Section \ref{barnsleysfernexample1} can now be written as,
\begin{equation}
	X^{x_0}_t(\omega) \,=\, f(X^{x_0},U_t(\omega)) \,=\, \left\{\begin{array}{lcl} G_1(X^{x_0}_{t-1}) & & \mathrm{if}~U_t(\omega)=1 \\
															 & \vdots&  \\
															G_K(X^{x_0}_{t-1}) & & \mathrm{if}~U_t(\omega)=K \end{array}\right. \label{ifsU}
\end{equation}
in line with (\ref{recursiveRandSysInputDef}) where $f:\mathbb{X}\times\mathbb{U}\rightarrow\mathbb{X}$ is non-random, time-invariant, measurable, and Lipschitz in $\mathbb{X}$.

We now consider a simplified version of Barnsley's fern. Let $\mathbb{X}=\mathbb{R}^2$, $x_0=0$, and $K=2$ and consider the affine transformations,
\begin{align}
	G_1(x) \,=\, \begin{bmatrix} 0.40& -0.3733\\ 0.060& 0.60\end{bmatrix}x+{\begin{bmatrix} 0.3533\\0.00\end{bmatrix}},\quad G_2(x) \,=\, \begin{bmatrix} -0.80& -0.1867\\ 0.1371& 0.80\end{bmatrix}x+{\begin{bmatrix} 1.10\\0.10\end{bmatrix}} \label{simpleFern}
\end{align}
with $\mathsf{p}=(0.2993, 0.7007)$. The induced trajectory of (\ref{ifsU}), (\ref{simpleFern}) plots a simplified fern \cite{DiaconisFreedman1999a}. 

Consider $N\in\mathbb{N}$ independent trajectories $\{\,(X^{x_0}_t,U_t)_n \,|\, t\in\{1,\ldots,T \},\,n\in\{1,\ldots,N\}\}$. We draw one set of samples for training a RNN; i.e. learning the network parameters (e.g. via back propagation with a mean square error cost \cite{PascanuMikolovBengio2013a}). The initial state is not part of the data in this example. We draw another set of samples for testing the RNN and computing mean error performance. We set $T_\mathrm{train}=50$, $N_\mathrm{train}=1000$, and $T_\mathrm{test}=10000$, $N_\mathrm{test}=2000$. We consider applications of the approximation over significantly longer time intervals than used in training. We train a simple network like (\ref{RNNspecialform}) with a single hidden layer of $6$ scalar activation functions $\sigma(\cdot):=\max\{0,\cdot\}$. See Figure \ref{fig:simpleFern}.

\begin{figure}[!ht]
    \centering
    \includegraphics[width=0.99\textwidth, height=0.145\textheight]{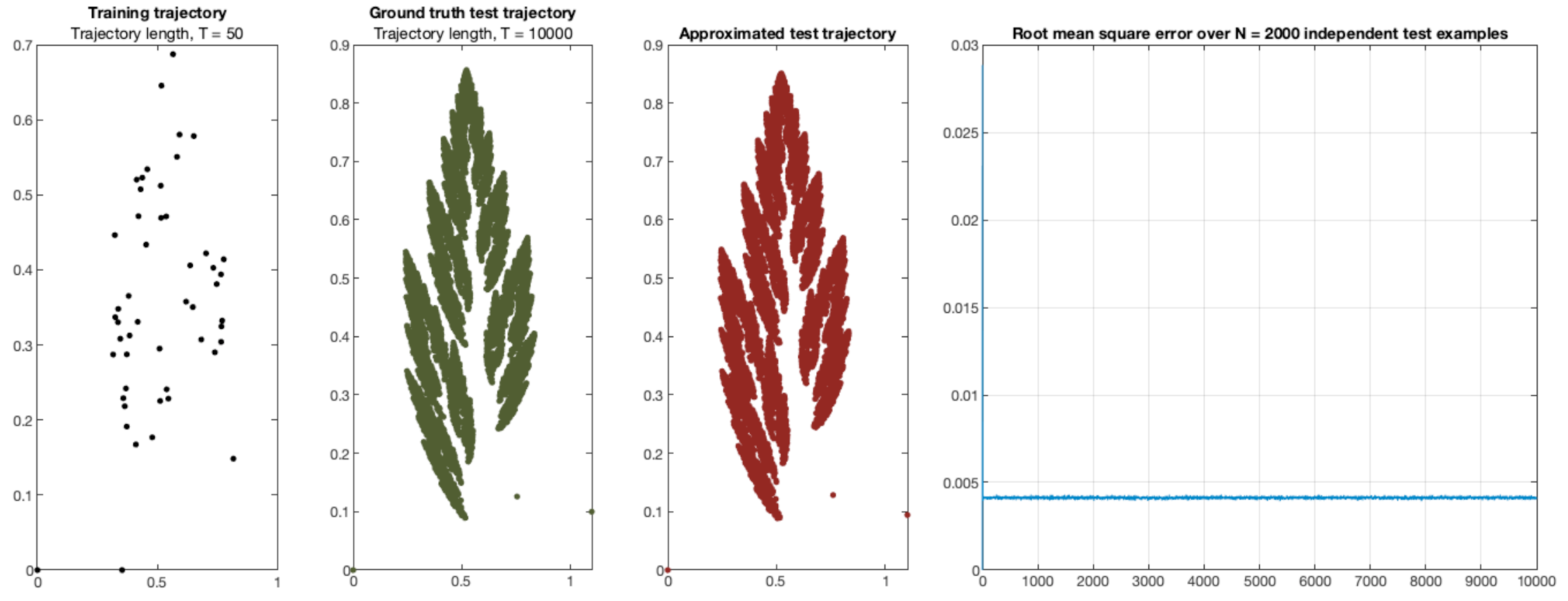}
    \caption{An example training, ground truth test and approximated test trajectory is shown; along with the root mean square error on $2000$ test examples, illustrating the time-uniform error control.}
    \label{fig:simpleFern}
\end{figure}

The approximation error introduced at each iteration does not accumulate, nor lead to increasing phase space distortion. This example just serves to illustrate a simple application of Theorem \ref{theoremuniversaltrajapprox}; i.e. we do not consider the training data or ergodic properties of this switched affine system here, nor seek state-of-the-art identification methods. We show later that even in very simple systems a small approximation error can grow and accumulate unbounded when the hypotheses of Theorem \ref{theoremuniversaltrajapprox} are violated.

\subsection{Generalisations to Asymptotically Stationary Sequences of Maps}

Firstly, we emphasise that Theorem \ref{theoremuniversaltrajapprox} is not restricted to Markov processes on $\mathbb{X}$. Asking just for stationarity of the random maps is a more general condition. We will show later a simple application of this result in optimal filtering wherein this additional generality is needed. We can say more. 

For the discussion here, lets say a process $\{U_t\,|\,t\in\mathbb{N}\}$ is \textit{asymptotically stationary} if the marginal law of $U_t$ converges toward an invariant measure, and if when $U_0$ is distributed according to this invariant measure, then $U_t$ is stationary. In at least some such cases the conclusions of Theorem \ref{theoremuniversaltrajapprox} remain true. We detail one specific situation here. From (\ref{recursiveRandSysInputDef}), we consider the special case
\begin{equation}
	F_{t}(x) \,=\, f(x,U_{t}) \,:=\, \mathbf{A}\,x \,+\, \mathbf{B}\,U_t  \label{recursiveAffineRandSysInputDef}
\end{equation}
where $\mathbf{A}$ and $\mathbf{B}$ are real finite linear maps $\mathbb{X}\mapsto\mathbb{X}$ and $\mathbb{U}\mapsto\mathbb{X}$. Suppose $U_t$ is generated by a linear state-space model of the form,
\begin{equation}
	U_t = \mathbf{G}\,U_{t-1} \,+\, U'_t \label{Uasymptoticstationary}
\end{equation}
for a linear $\mathbf{G}:\mathbb{U}\rightarrow\mathbb{U}$. Suppose $U'_t:\Omega\rightarrow\mathbb{U}$ is a stationary process independent of $U_0:\Omega\rightarrow\mathbb{U}$, and both have finite $p\geq1$ moments. For a general starting point, the process $\{U_t\,|\,t\in\mathbb{N}\}$ is not stationary. 

\begin{corollary}\label{corollarylinear}
	Consider the map $f:\mathbb{X}\times\mathbb{U}\rightarrow\mathbb{X}$ of (\ref{recursiveAffineRandSysInputDef}) and a RNN $\widehat{f}$ of the special form (\ref{RNNspecialform}), see also (\ref{genericspecialformRNN}). Let $X_0:\Omega\rightarrow\mathbb{X}$ be random, independent, and have finite $p\geq1$ moments. Assume $\{U_t\,|\, t\in\mathbb{N}\}$ is generated by \eqref{Uasymptoticstationary}. Assume the spectral radius of both $\mathbf{A}$ and $\mathbf{G}$ is strictly less than one. Then for any $\epsilon>0$ there is a finite $\widehat{f}$ and a random initialisation $H_0:\Omega\rightarrow\mathbb{H}^{d_h}$ such that,
\begin{equation}
		\mathsf{E}\big[ \,\big\| \widehat{F}_{t}\cdots\widehat{F}_2\cdot\widehat{F}_1(H_0) \,-\, {F}_{t}\cdots{F}_2\cdot {F}_1(X_0) \big\|^p\, \big]^{{1}/{p}}  ~\leq~ \epsilon, \qquad \forall t\in\mathbb{N}
\end{equation}
\end{corollary}

Proof is given in Appendix \ref{corollarylinearproof}. More general settings are possible; i.e. accommodating some notion of asymptotic stationarity of the random maps $\{\widehat{F}_t \,|\, t\in\mathbb{N}\}$. We leave this generalisation to the reader.

\section{Discussion}

The result in this article is closely related to an application of the Alekseev-Gr\"obner formula \cite{Iserles2009a} in discrete-time, applied to a process that is in some sense stable. Classical feedforward universal approximations then provide control over the local error in the Alekseev-Gr\"obner formulation. This interpretation is exactly that employed in \cite{HansonRaginsky2020a} for approximation of stable ordinary differential equations with bounded inputs using continuous-time RNNs; the Alekseev-Gr\"obner formula being key in \cite{HansonRaginsky2020a}. This trick is not uncommon. For example, similar steps as given are used to prove the time-uniform boundedness of approximation errors in sequential Monte Carlo methods \cite{DelMoralGuionnet2001a}. Viewing the RNN as a perturbed version of some ground truth dynamical system, the analysis herein reinforces the generally applied notion that exponential stability is robust, i.e. it is respected by small perturbations \cite{Sastry1999a,Iserles2009a}. Exponential ergodicity may also be preserved by small perturbations of Markov chain transition functions \cite{RobertsRosenthalSchwartz1998a}. 

One key trick used in the proof of Theorem \ref{theoremuniversaltrajapprox} is the telescopic expansion of the error. When combined with a stability property and control over the local, one-step, map approximation error, the result follows easily. Relevant here is that this approach, unlike Alekseev-Gr\"obner, works in a very general setting; e.g. with discrete-time random dynamical systems with random inputs. A second key insight is that the (average) contractive property that ensures the telescopic sum is convergent is the same property that leads to (asymptotic) stationarity in $\mathbb{X}$ and permits control over all the single-step errors.

This work adds to existing related literature. We consider random dynamical systems on non-compact domains with unbounded random inputs. The modelling used here is amenable to practice in many applications, e.g. time-series modelling. Moreover, the network constructions employed are based on practical deep RNNs with simple architectures; and the proof is straightforward so as to facilitate a practical understanding of its workings. We comment on the problem of finite-horizon trajectory approximation later. Prior related work on universal \textit{time-uniform} trajectory approximation, e.g. \cite{Sandberg1991a,Matthews1993a,MatthewsMoschytz1994a,GrigoryevaOrtega2018b,GrigoryevaOrtega2018a,HansonRaginsky2020a,HartHookDawes2020a,GononOrtega2021a}, is focused on deterministic dynamical systems on compact domains with bounded inputs. The article \cite{GononOrtega2020a} is an exception and accommodates unbounded stochastic inputs in the context of universal approximation with echo state networks (see \cite{Jaeger2001a,MaassNatschlaegerMarkram2002a,Lukosevicius2012a,YildizJaegerKiebel2012a}) which are beyond the focus of this work. Also, the conditions in \cite{GononOrtega2020a} are somewhat opaque when applied to the discrete-time random dynamical systems considered herein. For example, we know that certain average contraction conditions, viz. Lemma \ref{eltonlemma}, are loosely necessary to transfer stationarity from the random input sequence to the induced random process. Moreover, here, the induced process is only asymptotically stationary; and the special case in Corollary \ref{corollarylinear} and the proof of Theorem \ref{theoremuniversaltrajapprox} further imply a generalisation to only asymptotically stationary inputs. The conditions proposed herein thus seem to be the natural conditions to frame this kind of result. The derivations given here seem more direct than in related work.

Different network structures, in continuous-time, with various universal approximation capabilities including in some cases time-uniform approximation are given in \cite{MaassSontag2000a,MaassNatschlaegerMarkram2002a,MaassJoshiSontag2007a}. These networks differ from the deep RNNs considered herein; and the modelling, conditions precedent, and approximation capability studied in those articles differs from the discrete-time dynamical systems-based, state-space formulation, presented in this article. We do not address these models further here.

The result presented here has practical relevance in RNN-based trajectory approximation; e.g. in training, design and applications. For example, the result suggests that RNN training sequence lengths may be short in some cases while applications in test time may run indefinitely, with no unbounded accumulation of approximation error. The sparse feedback structure in the deep RNN architecture in the proof of Theorem \ref{theoremuniversaltrajapprox} offers insight into practical network constructions and may simplify design. For example, the notion of memory seems less critical in controlling long-time approximation errors in the problems mainly considered in this work. In applying RNNs, the main result implies a means for testing certain properties of the underlying process approximated; e.g. if performance degrades continually over long time horizons, then the underlying process may violate the considered hypotheses.

In discrete-time, the articles \cite{JinNikiforukGupta1995a,SchaeferZimmermann2007a} consider trajectory approximation on fixed finite time intervals. Again, the basic idea in \cite{JinNikiforukGupta1995a,SchaeferZimmermann2007a} is that the initial map approximation error is fixed small enough so that the inevitable accumulation of error over the finite interval remains below a desired threshold. Conversely, in the proof of Theorem \ref{theoremFiniteTimeTrajApprox} in Appendix \ref{appendixFiniteTime}, we construct a network for finite-horizon trajectory approximation for very general (random) dynamical systems with inputs. The network we construct has a feedback structure that is not aimed at recursion in the state variables, but rather at memorising the history of inputs and approximating different composed maps at each time step. The error in this formulation does not necessarily accumulate in time; but rather the network size grows significantly as a function of the time interval. This is a very different approach than considered in \cite{JinNikiforukGupta1995a,SchaeferZimmermann2007a} and is of interest on its own. Returning to the practice of network design, it may be of interest to consider the advantages of feedback structures between those based on state recursions and those based on memorisation. For example, it may be possible to improve performance in some systems when only training data is available and the stability/time-delay/dependency nature of the model is unknown. It may simplify modelling of systems that are marginally average contractive. It may reduce overall network size.

We apply an idealised universal approximation theorem in the space of Lipschitz maps \cite{Eckstein2020a,NeumayerGoujonBohraEtAl2022a}. In practice, learning Lipschitz maps with accurate Lipschitz constants or even constrained Lipschitz constants is hard \cite{HusterChiangChadha2018a,AnilLucasGrosse2019a,CohenHusterCohen2019a} and even estimating the Lipschitz constant of a neural network is generally intractable, see \cite{VirmauxScaman2018a}. Constraining or estimating the Lipschitz constant of a neural network, typically feedforward, has many applications (see \cite{HusterChiangChadha2018a,VirmauxScaman2018a,AnilLucasGrosse2019a,CohenHusterCohen2019a}), and in some sense we provide another.

Finally, in this article we consider a general setting amenable to applications. We consider random dynamical systems with unbounded random inputs on non-compact domains. We provide a stepped-through pedagogical proof that is simple and highlights where problems in long-time approximation may arise. There are natural extensions. There may be relaxations or variations of interest on the type of average contraction condition; noting the importance of Lemma \ref{eltonlemma} is in the establishment of some kind of asymptotic stationarity in the induced process. An average exponential rate of contraction is seemingly necessary, so that the telescopic sum in the proof is convergent. A slower average rate will not yield a convergent sum; which is consistent with the general stability theory of perturbed dynamical systems, see \cite{Sastry1999a,Iserles2009a}, and approximation results in \cite{HansonRaginsky2020a}. We only consider the rectified linear activation function in this article. Generalisations of the activation function as in \cite{HornikStinchcombeWhite1989a,LeshnoLinPinkusEtAl1993a,Eckstein2020a,KidgerLyons2020a} are a direction for possible extension. Different types of nonlinear system may also be considered, e.g. \cite{HansonRaginskySontag2021a}.

\section{Additional Illustrative Examples}

\subsection{Universal Time-Uniform Approximation of the Kalman Filter}

Let $\mathbb{Z}\subseteq\mathbb{R}^{d_z}$. We consider the following discrete-time state-space model,
\begin{align}
	Z_{t}^{\mu_0} \,&=\, \mathbf{G}\,Z_{t-1}^{\mu_0} \,+\, E_t \label{sysLintheory} \\
	U_{t} \,&=\, \mathbf{H}\,Z_{t}^{\mu_0} \,+\, D_t \label{measLintheory}
\end{align}
on $\mathbb{Z}\times\mathbb{U}$ where $\mathbf{G}$ and $\mathbf{H}$ are real matrices, and where\footnote{We denote a Gaussian distribution by $\mathcal{N}(\cdot,\cdot)$ with mean given by the first parameter and covariance by the second.} $Z_{0}\sim\mathcal{N}(\overline{Z}_0,\mathbf{C}_0)=:\mu_0$ where $\mathbf{C}_0$ is a positive-semi-definite covariance matrix. The random variables $E_t\sim\mathcal{N}(0,\mathbf{Q})$ and $D_t\sim\mathcal{N}(0,\mathbf{R})$ on $(E_t,D_t)\in\mathbb{Z}\times\mathbb{U}$ are independent, for all $t\in\mathbb{N}$, and of $Z_0$, with $\mathbf{R}$ positive definite. 

The process $Z_{t}$ is an underlying signal of interest that is observed by a sensor which produces the signal $U_t$. The Kalman filter for the model (\ref{sysLintheory}), (\ref{measLintheory}) is given by the two equations,

\noindent\begin{minipage}{.495\textwidth}
\begin{align}
	\overline{Z}_{t} &\,:=\ \mathsf{E}\left[ Z_t \,|\, U_1, \cdots, U_t\right] \nonumber\\ 
	&\,=\,  \mathbf{G}\overline{Z}_{t-1} \,+\, \mathbf{K}_{t}^*\left(U_{t} - \mathbf{H}\mathbf{G}\overline{Z}_{t-1}\right) \label{kalmanstateestimator} \\
	&\,=\,  \left(\mathbf{G} - \mathbf{K}_{t}^*\mathbf{H}\mathbf{G} \right) \overline{Z}_{t-1} \,+\, \mathbf{K}_{t}^* U_{t} \nonumber ~\\\nonumber
\end{align}
\end{minipage}
\begin{minipage}{.495\textwidth}
\begin{align}
	\mathbf{C}_{t} &\,:=\ \mathsf{Cov}\left[ Z_t \,|\, U_1, \cdots, U_t\right] \nonumber\\
	&\,=\, \mathsf{Cov}[ Z_t  - \widehat{Z}_{t} ] \label{kalmancovestimator} \\
	&\,=\, \left(\mathbf{I} - \mathbf{K}_{t}^*\mathbf{H} \right)\left(\mathbf{G}\mathbf{C}_{t-1}\mathbf{G}^\top + \mathbf{Q}\right) \nonumber ~\\\nonumber
\end{align}
\end{minipage}

\noindent with the optimal Kaman gain matrix,
\begin{equation}
	\mathbf{K}_{t}^* \, =\,  \left(\mathbf{G}\mathbf{C}_{t-1}\mathbf{G}^\top + \mathbf{Q}\right)\mathbf{H}^\top\,\left[\mathbf{R} + \mathbf{H}\left(\mathbf{G}\mathbf{C}_{t-1}\mathbf{G}^\top + \mathbf{Q}\right)\mathbf{H}^\top\right]^{-1} \label{kalmangain1}
\end{equation}

The Kalman filter is time-invariant and closed in the sense that the mapping,
\begin{equation}
	(\overline{Z}_{t},\mathbf{C}_{t})  \,=\, f(\overline{Z}_{t-1}, \mathbf{C}_{t-1},U_t) \,=:\, {F}_{t}(\overline{Z}_{t-1}, \mathbf{C}_{t-1}) \label{kalmanmapping}
\end{equation}
does not depend on $t\in\mathbb{N}$ and depends only on the static model parameters $(\mathbf{G}, \mathbf{H}, \mathbf{Q}, \mathbf{R})$ and the current-time observational input $U_t$, i.e. it also does not depend on any recursive computation of higher-order moments, etc. We define the filter state $X_t:=(\overline{Z}_{t},\mathbf{C}_{t})\in\mathbb{X}$ via a slight abuse of notation.

\begin{corollary} \label{linearneuralFilterTheorem}
	Consider a linear-Gaussian signal and observation model (\ref{sysLintheory}), (\ref{measLintheory}). Assume the spectral radius of $\mathbf{G}$ is strictly less than one. Let $\widehat{f}:\mathbb{X}\times\mathbb{U}\rightarrow\mathbb{X}$ denote a RNN of the special form (\ref{RNNspecialform}), see also (\ref{genericspecialformRNN}), taking as input observations in the ordered sequence $\{{U_t}\,|\,t\in\mathbb{N}\}$. Then for any $\epsilon>0$ there exists a finite $\widehat{f}$ and an initialisation $\widehat{x}_0\in\mathbb{X}$ such that,
\begin{equation}
	 \mathbb{E}\left[\, \left\| \, X_t^{x_0}\, -\, \widehat{F}_{t}\cdots\widehat{F}_2\cdot\widehat{F}_1(\widehat{x}_0)\,\right\|^2\, \right]^{{1}/{2}} \,\leq\,\, \epsilon,~\qquad \forall\,t\in\mathbb{N},~~\forall x_0\in\mathbb{X}
\end{equation}
\end{corollary}

\begin{proof}
This is now a straightforward application of Corollary \ref{corollarylinear} after appropriate notational identifications and use of some basic properties of the optimal Kalman filter. We omit the details. See \cite{BishopBonilla2022}.
\end{proof}

We remark that the observation sequence is neither Markov nor stationary in general, and the hypothesis of the corollary do not call for these conditions. 

The topic of recurrent neural networks and universal approximation of general Bayesian filters (including time-uniform results) is explored further in \cite{BishopBonilla2022}, with additional examples and empirical results (also contrasting the time-uniform result of Corollary \ref{linearneuralFilterTheorem} with the finite-horizon result in Theorem \ref{theoremFiniteTimeTrajApprox}).

\subsection{Simple Counterexamples with a Discrete-Time Ornstein–Uhlenbeck Process}

Let $\mathbb{X}=\mathbb{R}$. Consider a very simple scalar system of the form,
\begin{equation}
	X^{\mu_0}_t \,=\, f(X_{t-1},U_t) \,:=\,  \varrho + \alpha(X_{t-1} - \varrho) + U_t \label{ar1true}
\end{equation}
where $\alpha\in\mathbb{R}$ is a deterministic parameter, and $U_t\sim\mathcal{N}(0,\varsigma^2)$, for all $t\in\mathbb{N}$. Let $\mu_0 = \mathcal{N}(\varrho_0,\varsigma_0^2)$. Then $X^{\mu_0}_t\sim\mathcal{N}(\varrho_t,\varsigma_t^2)$ is Gaussian for all $t\in\mathbb{N}$ and,
\begin{align}
		\varrho_t \,=\, (1-\alpha^t)\,\varrho + \alpha^t\,\varrho_0 \qquad\mathrm{and}\qquad \varsigma_t\,=\, \alpha^{2t}\varsigma^2_0 + \frac{1-\alpha^{2t}}{1-\alpha^{2}}\varsigma^2
\end{align}
The hypotheses of Lemma \ref{eltonlemma} and Theorem \ref{theoremuniversaltrajapprox} hold with $|\alpha|<1$ and $\{X^{\mu_0}_t \,|\,t\in\mathbb{N}\}$ is asymptotically stationary with Gaussian invariant measure $\mu_\infty:=\mathcal{N}(\varrho,\varsigma^2/(1-\alpha^2))$. If $|\alpha|>1$ the process is unstable.  

Consider the simple feedback RNN of (\ref{RNNspecialform}), see also (\ref{genericspecialformRNN}). Let 
\begin{align}
	\widehat{X}^{X_0}_t \,=\, \widehat{f}(\widehat{X}_{t-1},U_t) \,:=\, \max\{0, U_t + \varrho - \alpha\varrho + b + \alpha\,\widehat{X}_{t-1}\} - b + \delta \label{ar1NN}
\end{align}
which is of the form (\ref{RNNspecialform}). The parameter $\delta\in\mathbb{R}$ will be used to introduce an approximation error (noting it is a very simplistic error). With $\delta=0$, picking the constant $b>0$ large enough ensures that for all practical purposes $\widehat{f}=f$, e.g. for all practical realisations of $\omega\in\Omega$ simulated. 

In Figure \ref{fig:samplesAR1} we plot sample trajectories for the process (\ref{ar1true}) with $\varrho_0=20$, $\varsigma_0^2=\varsigma^2=1$ and $\alpha\in\{0.99,1,1.001\}$. We plot with each sample trajectory the empirical distribution of the trajectory. 

\begin{figure}[!ht]
    \centering
    \includegraphics[width=0.975\textwidth]{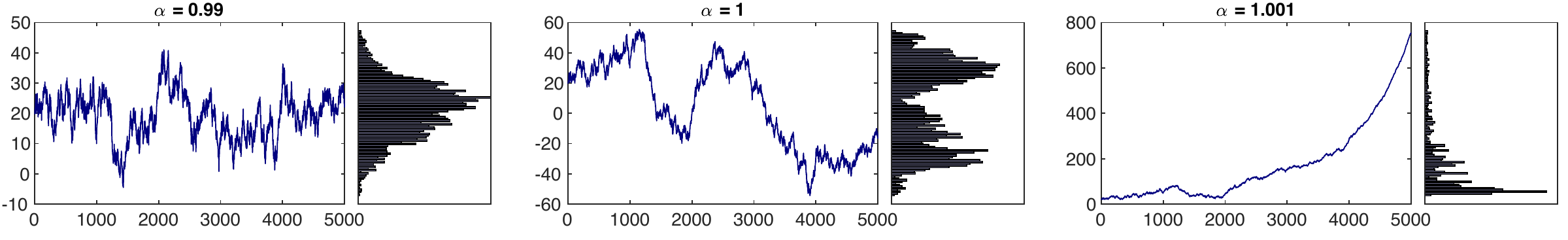}
    \caption{Sample trajectories for a range of $\alpha$ values showing the corresponding stability/stationarity. }
    \label{fig:samplesAR1}
\end{figure}

Now we introduce some error $|\delta|>0$ into the approximation (\ref{ar1NN}) and compare trajectories with (\ref{ar1true}). We initialise $\widehat{X}_0={X}_0$ almost surely to emphasise the results of Theorem \ref{theoremuniversaltrajapprox}. We plot in Figure \ref{fig:rmseAR1} the empirical root mean squared error averaged over $5000$ realisations of the random input for all $\delta$ ranging from $0.005$ to $0.1$ in $0.005$ increments. The results are unsurprising and consistent with Theorem \ref{theoremuniversaltrajapprox} and with the theory of perturbed stable/unstable dynamical systems \cite{Sastry1999a,Iserles2009a}.

\begin{figure}[!ht]
    \centering
    \includegraphics[width=0.975\textwidth]{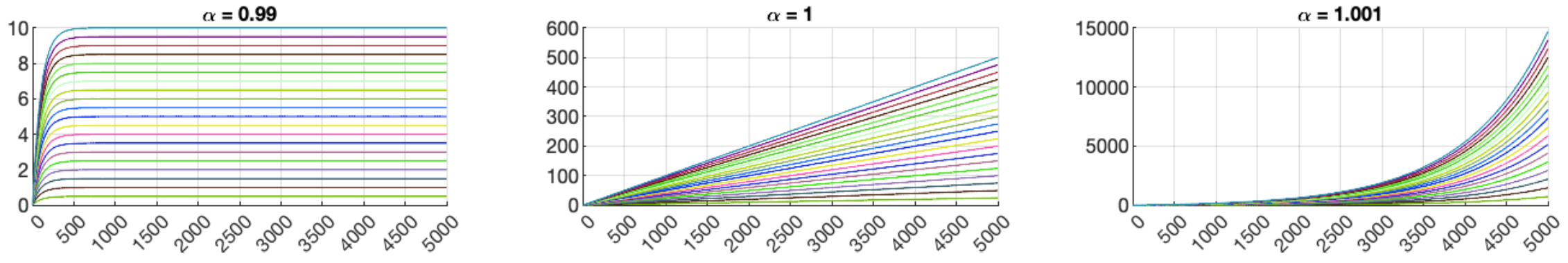}
    \caption{Root mean square errors. }
    \label{fig:rmseAR1}
\end{figure}

We have considered the case in which the random maps are stationary and given a counterexample when the sequence is not exponentially $p$-contractive (i.e. when $|\alpha|>1$). To see where things may go wrong when stationarity is violated, consider just the case in which $|\alpha|<1$ and let $U_t = \beta \,U_{t-1} + U'_t$ as in (\ref{Uasymptoticstationary}) with $U_t'\sim\mathcal{N}(0,\varsigma^2)$ for all $t\in\mathbb{N}$. The sequence of random maps is exponentially $p$-contractive for any finite $\beta\in\mathbb{R}$. Consider now the identity approximation of $U_t$ given by $\sigma(U_t + b)-b$, see (\ref{ar1NN}), with some large bias $b>0$. If $|\beta|<1$, then Corollary \ref{corollarylinear} applies as $U_t$ is asymptotically stationary; in particular, there is a $b=b(\varepsilon)>0$ such that $\mathsf{E}[\|U_t - (\sigma(U_t + b)-b)\|]\leq\varepsilon$ for any $\varepsilon>0$. However, with $|\beta|>1$, and if $U_t$ moves to the negative, then there is a time $s= s(b)\in\mathbb{N}$ such that $\forall t>s$ we have $\sigma(U_t + b)-b=-b$. The approximation error for all $t>s$ grows on average and a time-uniform bound is impossible (even though the sequence of random maps is exponentially $p$-contractive). 

Stationarity and exponential contractivity seem the two natural conditions under which to consider time-uniform results like Theorem \ref{theoremuniversaltrajapprox}/Corollary \ref{corollarylinear}, at least on unbounded domains with unbounded inputs. 

We remark for relevance that the Kalman filter state estimator equation (\ref{kalmanstateestimator}) resembles an Ornstein–Uhlenbeck process with the observation process acting as the random driver. See also \cite{BishopBonilla2022} for empirical inference results with RNNs and Kalman filters, contrasting the time-uniform result of Theorem \ref{theoremuniversaltrajapprox}/Corollary \ref{linearneuralFilterTheorem} with the finite-horizon result of Theorem \ref{theoremFiniteTimeTrajApprox}.

\bibliographystyle{apalike}
\bibliography{paper}

\appendix

\section{Approximation Capabilities of Feedforward Networks}\label{UAFFappendix}

Let $\sigma(\cdot):=\max\{\cdot,0\}:\mathbb{R}\rightarrow\mathbb{H}\subseteq\mathbb{R}$; which acts component wise on vector $v\in\mathbb{R}^d$ so $\sigma(v)\in\mathbb{H}^d$. 

A deep feedforward neural network, with $L\in\mathbb{N}$ layers, is a special case of the RNN in (\ref{RNNspecialform}), see also (\ref{genericspecialformRNN}), and is given here in the form,
\begin{equation}
\widehat{f}:\mathbb{V}\rightarrow\mathbb{W}\qquad\Longleftrightarrow\qquad \widehat{f}(\cdot)\,:=\, \tau_L\cdot\sigma\cdot\tau_{L-1}\cdots\sigma\cdot\tau_{2}\cdot\sigma\cdot\tau_1(\cdot) 
\label{FNNform}
\end{equation}
where $\tau_l(\cdot)$ is a finite affine map and $\mathbb{V}\subseteq\mathbb{R}^{d_v}$ and $\mathbb{W}\subseteq\mathbb{R}^{d_w}$.

Universal approximation theorems for (deep) feedforward networks state that certain classes of functions over particular domains can be approximated to any desired accuracy by large, but relatively simple, compositions of the form (\ref{FNNform}). We state the following specific result used in this work.

\begin{lemma}\cite{ParkYunLeeEtAl2021a,KidgerLyons2020a,Eckstein2020a} \label{lemmaUALp}
	Let $V:\Omega\rightarrow\mathbb{V}$ be random with finite $p$-th moments, $p\geq1$. Assume $V\mapsto f(V)$ is measurable with finite $p$-th moments. For any $\epsilon>0$ there is a finite $\smash{\widehat{f}}$ in (\ref{FNNform}) with,
\begin{equation}
	 \mathsf{E}\big[\, \big\| \widehat{f}(V) - f(V)  \big\|^p \,\big]^{{1}/{p}} \,\leq\, \epsilon \label{lpfeedforwarderror}
\end{equation}
if and only if the width of each hidden layer is $\geq \max\{d_v+1,d_w\}$.

If $f$ is also Lipschitz with Lipschitz constant $\mathfrak{L}\in(0,\infty)$, then for any $\epsilon>0$ there is a finite $\widehat{f}$ of the form (\ref{FNNform}) such that (\ref{lpfeedforwarderror}) holds and also $\widehat{f}$ is Lipschitz with Lipschitz constant $\leq\mathfrak{L}$.
\end{lemma}

The first part is a standard universal approximation theorem for a measurable map with finite $p$-th moments; but with the added optimally minimal layer width. This result is given in \cite{ParkYunLeeEtAl2021a}. We point also to the nice exposition in \cite{KidgerLyons2020a}. We also note \cite{HornikStinchcombeWhite1989a,LeshnoLinPinkusEtAl1993a} which considers a single (wide) hidden layer, different activation functions, and measures with compact support. 

Approximation in the space of Lipschitz functions in the second part is studied in \cite{Eckstein2020a}, see also \cite{NeumayerGoujonBohraEtAl2022a}, for functions on compact domains and with approximation in a uniform norm. This can be extended to non-compact domains and integral norms by following the proof in \cite[Theorem 4.16]{KidgerLyons2020a} and substituting \cite[Theorem 1]{Eckstein2020a} for \cite[Proposition 4.9]{KidgerLyons2020a} in the proof of \cite[Theorem 4.16]{KidgerLyons2020a}. The steps are rather straightforward. We do not specify that approximation is possible with minimal width in this case.

\section{Universal Trajectory Approximation on Finite Time Intervals}\label{appendixFiniteTime}

In this section we provide a different proof of a universal approximation theorem for trajectories of random dynamical systems over finite horizons. We use here a drastically different RNN architecture than considered for time-uniform universal approximation, and considered in prior work on finite-horizon trajectory approximation cf. \cite{JinNikiforukGupta1995a,SchaeferZimmermann2007a}. The result in this section requires minimal assumptions on the dynamical system; in particular, no contractiveness properties are required (i.e. the system may be unstable) and no stationarity assumptions are required. This result complements those in \cite{JinNikiforukGupta1995a,SchaeferZimmermann2007a}.

\begin{theorem} \label{theoremFiniteTimeTrajApprox}
	Consider the map $f:\mathbb{X}\times\mathbb{U}\rightarrow\mathbb{X}$ of (\ref{recursiveRandSysInputDef}) and the RNN $\widehat{f}$ in (\ref{RNNform1}), (\ref{RNNform2}), (\ref{RNNform3}), see also (\ref{genericformRNN}). Let $X_0:\Omega\rightarrow\mathbb{X}$ be random with finite $p\geq1$ moments and assume $\mathsf{E}\left[\|F_t(x)-x\|^p\right]<\infty$. Then for any $\epsilon>0$ and any finite $T\in\mathbb{N}$ there is a finite $\widehat{f}$ and a random initial $H_0:\Omega\rightarrow\mathbb{H}^{d_h}$ such that,
\begin{equation}
		\mathsf{E}\big[ \,\big\| \widehat{F}_{t}\cdots\widehat{F}_2\cdot\widehat{F}_1(H_0) \,-\, {F}_{t}\cdots{F}_2\cdot {F}_1(X_0) \big\|^p\, \big]^{{1}/{p}}  ~\leq~ \epsilon
\end{equation}
for all $t\in\{0,\ldots,T\}$.
\end{theorem}

The assumptions in the preceding theorem are mild. The model (\ref{recursiveRandSysInputDef}) is based on Lipschitz maps, not necessarily contractive, as these were the main focus of this article. However, in the preceding theorem this can be relaxed if $f$ preserves $p$-integrability. The horizon $T$ can be any finite non-negative integer, but the size of the RNN constructed in the proof grows with $T\in\mathbb{N}$.

The proof given subsequently does not rely on Gr\"onwall-Bellman-type inequalities and the error at every step along the finite interval can potentially be bounded with similar precision; i.e. the error at any individual time is not necessarily growing along the interval with the network construction used to prove Theorem \ref{theoremFiniteTimeTrajApprox}. This contrasts with the constructions used to prove the main results in \cite{JinNikiforukGupta1995a,SchaeferZimmermann2007a} which rely on bounding the approximation error small enough at the initial step along a finite time interval, and then allowing this error to grow (typically exponentially) so that at the terminal time $T$ of the finite interval the error remains below the desired error precision threshold (with an application of simplistic Gr\"onwall-Bellman inequalities). Qualitatively, the result/proof of Theorem \ref{theoremFiniteTimeTrajApprox} and those in, e.g., \cite{JinNikiforukGupta1995a,SchaeferZimmermann2007a} suggest different approximation capabilities (due to differing network constructions in the proofs) even though the statement of the results in each case are mostly the same.

\begin{proof}

We consider a special form of (\ref{RNNform1}), (\ref{RNNform2}), (\ref{RNNform3}) in which feedback is solely from the first hidden layer to itself. This is in contrast to (\ref{RNNspecialform}) where feedback is from the last layer, or equivalently the last hidden layer, to the first hidden layer and thus practically a recursion on $\widehat{X}_t$, see also (\ref{genericspecialformRNN}). In this proof we will use a wide hidden first layer with feedback to itself as a kind of memory bank for storing the external inputs over time. This structure was originally applied in \cite{Lo1994a} to prove a RNN-based approximation result for optimal filtering on finite time intervals. 

With that said, we write the RNN form as,
\begin{align}
	\widehat{X}_t \,=\, \tau_L\cdot\sigma\cdot\tau_{l-1}\cdots\sigma\cdot\tau_2'\cdot\sigma(\tau_1(U_t)  + \varphi(h_{1,t-1}))  \label{RNNformfixedhorizon}
\end{align}
with notation carrying forward from (\ref{RNNform1}), (\ref{RNNform2}), (\ref{RNNform3}) and where $\varphi:\mathbb{R}^{d_{h_1}}\rightarrow\mathbb{R}^{d_{h_1}}$ is a finite affine feedback map.

We will consider the structure of the affine maps $(\tau_1 +\varphi)$ and of $\tau_{1.5}$ in the re-defined $\tau'_2:=\tau_{2}\cdot\tau_{1.5}$; noting the sum of two affine maps is affine, as is the composition so no generality is lost.
 
We will write these maps as,
\begin{equation}
	\tau_1:\mathbb{U}\rightarrow\mathbb{H}^{d_{h_1}}, \qquad u ~\mapsto~ \tau_1(u) \,:=\, \mathbf{M}_{\tau_1}\,u  + \mathbf{b}_{\tau_{1}} 
\end{equation}
\begin{equation}
	\varphi:\mathbb{H}^{d_{h_1}}\rightarrow\mathbb{H}^{d_{h_1}}, \qquad  h ~\mapsto~ \varphi(h) \,:=\, \mathbf{M}_\varphi\,h   
\end{equation}
and 
\begin{equation}
	\tau_{1.5}:\mathbb{H}^{d_{h_1}}\rightarrow\mathbb{H}^{d_{h_{1.5}}}=\mathbb{H}^{d_{h_{1}}}, \qquad  h ~\mapsto~ \tau_{1.5}(h) \,:=\, h + \mathbf{b}_{\tau_{1.5}} 
\end{equation}
We will use this one and a half layer representation to construct a new input for a feedforward neural network. The new input will consist of a time keeper, the initial state for the dynamical system to be simulated, and a memory of all external inputs received. 

To gain some intuition on how we proceed, as per \cite[Theorem 4.16]{KidgerLyons2020a}, note that for some scalar random variable $U:\Omega\rightarrow\mathbb{R}$, the identity approximation $(\max\{0,U+b\}-b)$ can always be made as close as desired to $U$ in the sense of $\mathsf{E}[|U - (\max\{0,U+b\}-b)|^p]$ with a sufficiently large bias $b>0$.

In particular, consider
\begin{align}
	h_{1.5,t} \,&=\, \tau_{1.5}(h_{1,t}) \\
	h_{1,t} &=\, \sigma(\tau_{1}(U_t)+\varphi(h_{1,t-1}))
\end{align}
We will design the first layer and a half to yield an output at any time $t\in\{1,\ldots,T\}$ in the form,
\begin{equation}
	h_{1.5,t} \,=\, \left(t, \widehat{x}_0, U_t, U_{t-1}, \ldots, U_1, 0, 0, \ldots, 0\right)
\end{equation}
where the $(U_t, U_{t-1}, \ldots, U_1, 0, 0, \ldots, 0)$ component is considered over the (sub)-domain of $\Omega$ on which $\mathsf{P}$ places most mass, as desired, see \cite[proof of Theorem 4.16]{KidgerLyons2020a}. And then similarly with,
\begin{equation}
	h_{1.5,T} \,=\, \left(T, \widehat{x}_0, U_T, U_{T-1}, \ldots, U_1 \right)
\end{equation}
With these dimensions in mind, to achieve this construction, let,
\begin{equation}
	\mathbf{M}_{\tau_1}\,=\left[\begin{array}{c} 	\left[0~0~\ldots~0\right]\\ 
									\mathbf{0}\\
									\mathbf{I} \\
										 \mathbf{0} \\
										  \vdots  \\
										  \mathbf{0} \end{array}\right],\qquad
	\mathbf{M}_{\varphi}\,=\left[\begin{array}{ccccccc} 1 &0~0 & 0~0 & 0~0 & \cdots & 0~0 & 0~0\\
										0 & \mathbf{I} & \mathbf{0} & \mathbf{0} & \mathbf{0} & \cdots & \mathbf{0}\\
										0 & \mathbf{0}& \mathbf{0} & \mathbf{0} & \mathbf{0} & \cdots & \mathbf{0}\\
										0 &  \mathbf{0}& \mathbf{I} & \mathbf{0} & \mathbf{0} & \cdots & \mathbf{0}\\
										 0 &  \mathbf{0} & \mathbf{0}&\mathbf{I} & \mathbf{0} & \cdots & \mathbf{0}\\
										 0 &  \mathbf{0}& \mathbf{0} &\mathbf{0} & \ddots & \ddots & \mathbf{0}  \end{array}\right]
\end{equation}
where $\mathbf{I}$ in $\mathbf{M}_{1}$ and all but the first $\mathbf{I}$ in $\mathbf{M}_{\varphi}$ denotes a $d_u$-dimensional identity matrix, and the first $\mathbf{I}$ in $\mathbf{M}_{\varphi}$ is a $d_x$-dimensional identity matrix. And let,
\begin{equation}
	 \mathbf{b}_{\tau_1} \,=\,(1,(0,\ldots,0),(b,\ldots,b),0,\ldots,0),\qquad
	 \mathbf{b}_{\tau_{1.5}} \,=\, (0,-b,\ldots,-b)
\end{equation}
Initialise $h_{1,0}=(0, \widehat{x}_0+(b,\ldots,b), (0,\ldots,0), b, \ldots, b)$. With this construction and with, lets say, $b>0$ large enough, we have the desired state $h_{1.5,t}$. 

There is no additional feedback in the RNN (\ref{RNNformfixedhorizon}) constructed in this proof. This state $h_{1.5,t}\in\mathbb{H}^{d_{h_{1.5}}}$ may be viewed as an input then for the following feedforward neural network,
\begin{align}
	\widehat{X}_t \,=\, \tau_L\cdot\sigma\cdot\tau_{l-1}\cdots\sigma\cdot\tau_2(h_{1.5,t})  \label{FFNNformfixedhorizon}
\end{align}

We consider a type of unfolded conditional functional,
\begin{equation}
	\widetilde{F}(t,\widehat{x}_0,U_1,\ldots,U_T) \,=\, \left\{\begin{array}{lll} F_1(\widehat{x}_0) &&\mathrm{if}~t=1 \\ F_2\cdot F_1(\widehat{x}_0) &&\mathrm{if}~t=2 \\ ~~~~\vdots &&~~~~ \vdots \\ F_{T}\cdots F_2\cdot F_1(\widehat{x}_0) &&\mathrm{if}~t=T \end{array}\right.
\end{equation}
Note that this function is (Borel) measurable as the composition of measurable functions. We can now apply classical universal approximation theorems on $\widetilde{F}$. For example, see Lemma \ref{lemmaUALp} in Appendix \ref{UAFFappendix}, or see \cite[proof of Theorem 4.16]{KidgerLyons2020a} for an easy to follow construction directly applicable in this setting. 
\end{proof}

\section{Proof of Corollary \ref{corollarylinear}}\label{corollarylinearproof}

Stacking (\ref{recursiveAffineRandSysInputDef}) and \eqref{Uasymptoticstationary} we may write, 
\begin{equation}
	\left[\begin{array}{c} X_t \\ U_t  \end{array}\right] \,=\, \left[\begin{array}{cc} \mathbf{A} & \mathbf{B}\mathbf{G} \\ \mathbf{0} & \mathbf{G}  \end{array}\right] \left[\begin{array}{c} X_{t-1} \\ U_{t-1}  \end{array}\right]  \,+\, \left[\begin{array}{cc} \mathbf{B} & \mathbf{0} \\ \mathbf{0} & \mathbf{Id}  \end{array}\right]  U'_t
\end{equation}
It is assumed the spectral radius of both $\mathbf{F}$ and $\mathbf{G}$ is strictly less than one. Hence, the transition matrix for the given stacked system has spectral radius less than one (its eigenvalues are the union of the eigenvalues of the diagonal blocks). The process $\{U'_t(\omega)=U'(\theta^t(\omega)\,|\,t\in\mathbb{N}\}$ is assumed stationary. The hypotheses of Lemma \ref{eltonlemma} are satisfied. Since $X_0$ and $U_0$ are assumed independent of $U'_0$ it follows that, with the obvious change of notation, the hypotheses of Theorem \ref{theoremuniversaltrajapprox} are satisfied for the stacked process, with $\{U'_t\,|\,t\in\mathbb{N}\}$ as the input in this case. Consider some RNN approximation, of the stacked system, in the form (\ref{RNNspecialform}) with $U'_t$ as input and $(\widehat{X}_t,\widehat{U}_t)$ as feedback. Suppose now that we just substitute $U_t$ for the input $U'_t$ in this network, and set to zero the affine feedback map on $\widehat{U}_t$. The result is a sub-RNN that approximates the desired map $F_{t}(x) = \mathbf{A}\,x + \mathbf{B}\,U_t$. By forcing an accurate enough RNN approximation of the stacked system, via Theorem \ref{theoremuniversaltrajapprox}, the conclusions of Corollary \ref{corollarylinear} hold for the sub-RNN approximation of the desired $F_{t}$ when changing inputs from $U'_t$ to $U_t$. In particular, with the norm $\|\cdot\|_p:=\mathsf{E}[\|\cdot\|^p]^{1/p}$ we have,
\begin{equation}
	\|U_t-U'_t  \|_p \,\leq\, \|\mathbf{G}^t\|\|U_{0}\|_p + \sum_{i=1}^{t-1} \|\mathbf{G}^{i}\| \|U_{i}' \|_p \,\leq\, \|\mathbf{G}^t\|\|U_{0}\|_p + \frac{\|U' \|_p}{1 - \|\mathbf{G}\|} 
\end{equation}
which is bounded by some finite constant; and $\|\mathbf{G}^t\|$ actually goes to zero exponentially fast.  \qed

\end{document}